\documentclass[11pt]{article}

\usepackage[preprint]{acl}

\usepackage{times}
\usepackage{latexsym}

\usepackage[T1]{fontenc}

\usepackage[utf8]{inputenc}

\usepackage{microtype}

\usepackage{inconsolata}

\usepackage{graphicx}

\usepackage{amssymb}       
\usepackage{amsthm}
\newtheorem{lemma}{Lemma}
\newtheorem{theorem}{Theorem}
\usepackage{algorithm}    
\usepackage{algpseudocode}
\usepackage{float}         
\usepackage{booktabs}   
\usepackage{multirow}   
\usepackage{array}        
\usepackage{amsmath}    

\title{Escaping Low-Dimensional Overlap: Multi-Task Model Merging via High-Dimensional Sparse Disentanglement}

\author{
  \textbf{Yihang Zhang}\textsuperscript{1}\thanks{Equal contribution.},
  \textbf{Shengke Sun}\textsuperscript{2}\footnotemark[1],
  \textbf{Junjie Wen}\textsuperscript{3},
  \textbf{Feng Zeng}\textsuperscript{1}\thanks{Corresponding author.} \\
  \textsuperscript{1}Central South University \\
  \textsuperscript{2}Nanjing University of Science and Technology \\
  \textsuperscript{3}Hefei University of Technology \\
  \texttt{zhangyihang1010@gmail.com},
  \texttt{sunshengke@njust.edu.cn},\\
  \texttt{junjie.wen@mail.hfut.edu.cn},
  \texttt{fengzeng@csu.edu.cn}
}

\begin{document}
\maketitle
\begin{abstract}
Model merging provides an efficient way to construct multi-task generalist models without additional training, but its performance often degrades under severe task interference. Task interference in model merging primarily stems from \textit{superposition}, where task-specific features become entangled within the parameter space. This entanglement renders conventional decomposition methods insufficient for effectively isolating useful task directions from interfering components.
In this paper, we propose a sparse-representation-based merging framework that uses Sparse Autoencoders (SAEs) to project task vectors into a high-dimensional sparse feature space, enabling feature-level disentanglement before fusion. To reduce computational overhead, we further introduce a lightweight Group-Ranked Zeroth-Order Optimizer (GR-ZOO) to identify task-critical layers for selective merging.
Experiments on both Qwen2.5-1.5B and Qwen2.5-7B demonstrate that our method consistently outperforms representative baselines, including Task Arithmetic, TIES-Merge, DARE, Fisher-Merge,and several recent training-free merging methods, across mathematical reasoning, code generation, instruction following, and general knowledge tasks. In a highly conflicting four-task setting on Qwen2.5-1.5B, our method further achieves a 2.78\% improvement over the strongest baseline.
\end{abstract}

\begin{figure*}[t]
  \centering
  \includegraphics[width=\textwidth]{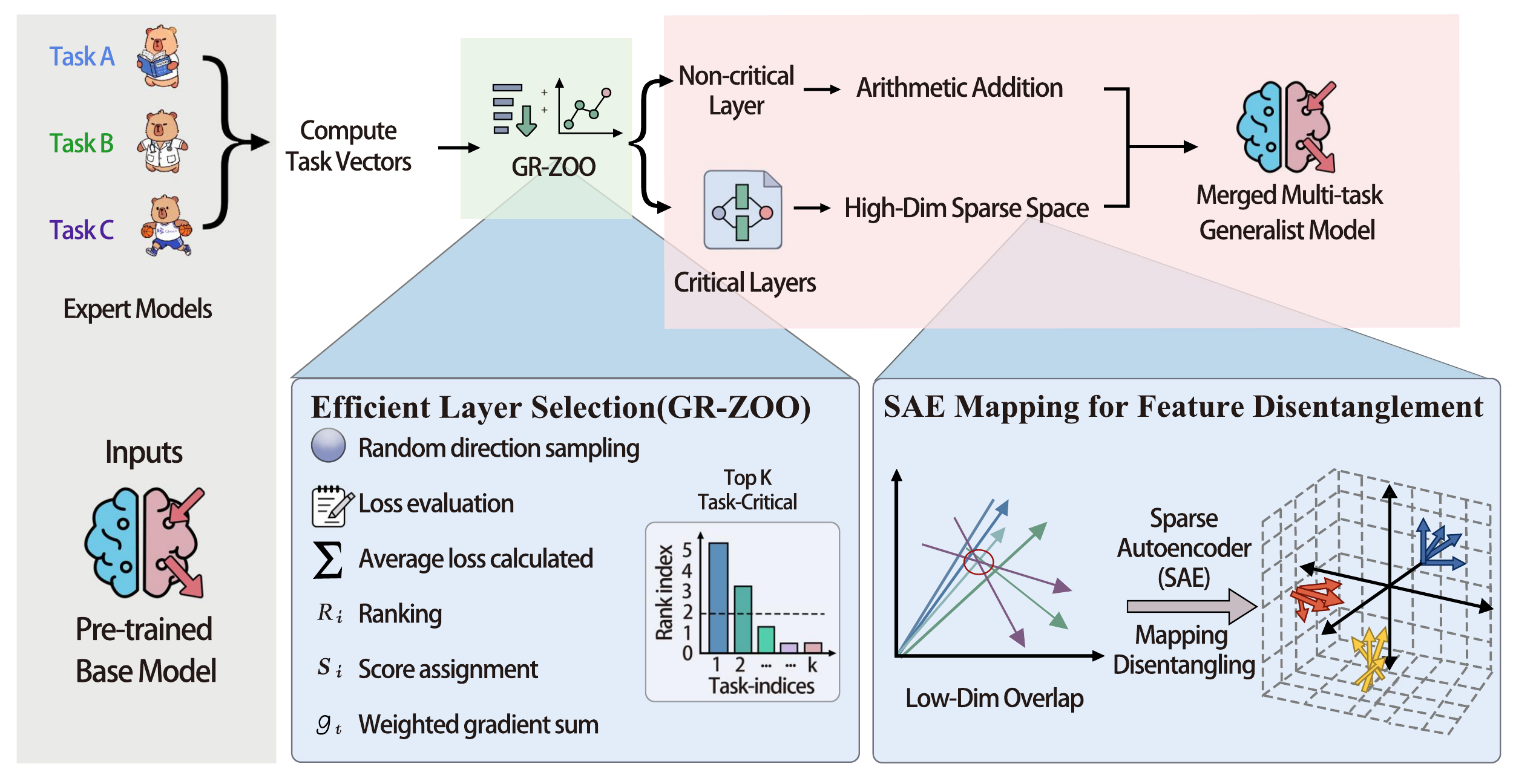}
  \caption{Overall architecture of our proposed High-Dimensional Sparse Disentanglement Merging framework. The pipeline first extracts task vectors from expert models and applies GR-ZOO to identify task-critical layers. Non-critical layers are merged through arithmetic addition, while critical layers are projected into a high-dimensional sparse space via SAE-based feature disentanglement, enabling interference-free fusion of overlapping task features.}
  \label{fig:architecture}
\end{figure*}

\section{Introduction}

As the scale of foundation models continues to grow~\cite{yang2025qwen3,team2023gemini,achiam2023gpt,guo2025deepseek}, fine-tuning and deploying separate models for different downstream tasks has become increasingly prohibitive~\cite{ouyang2022training,Tang_2025_CVPR}. Recently, model merging~\cite{ilharco2022editing,wortsman2022model} has emerged as an efficient way to integrate multiple task-specific models without additional training. It aims to construct a versatile generalist model by directly fusing the weights of multiple expert models that excel in individual tasks~\cite{matena2022merging,ilharco2022editing,yadav2023ties,yu2024language}.

A major challenge in model merging lies in task interference~\cite{jin2023dataless,yadav2023ties,tam2024merging}. Due to different fine-tuning data distributions and optimization objectives of expert models, their parameter updates may encode conflicting task-specific directions. As a result, the merged model often struggles to simultaneously preserve the capabilities of all constituent experts. To mitigate such interference, numerous model merging approaches have been proposed. Parameter-space arithmetic methods, such as TIES-Merge~\cite{yadav2023ties} and DARE~\cite{yu2024language}, attempt to reduce conflicts by pruning, rescaling, or resolving sign conflicts of task vectors before fusion. Parameter decomposition methods, such as TSV~\cite{gargiulo2025task} and Twin-Merging~\cite{lu2024twin}, further seek to separate task-relevant and interfering components through low-rank or structured decompositions. However, these methods still rely on operations within the original parameter space, and thus remain limited when task-specific updates are highly entangled.

In this paper, we attempt to alleviate this limitation through the lens of \textit{superposition}, a phenomenon widely discussed in mechanistic interpretability~\cite{elhage2022toy,cunningham2023sparse}. Mechanistic interpretability studies reveal that in neural networks, multiple functionally distinct features may share non-orthogonal directions in the same representation space. In model merging, this view suggests that task vectors are not always clean combinations of independent task-specific components~\cite{cunningham2023sparse,bricken2023monosemanticity}. Instead, updates from different tasks may partially reuse, compete for, or interfere with overlapping feature directions. Therefore, operations performed only in the original parameter space, such as direct arithmetic or orthogonal decomposition, may still suffer from interference when task-specific features are highly entangled.

Building upon this observation, we propose to reduce task interference by projecting task vectors into a higher-dimensional sparse feature space before fusion, thereby alleviating merging conflicts. We propose \textbf{High-Dimensional Sparse Disentanglement Merging}, a sparse-representation-based framework for model merging. Our method uses Sparse Autoencoders (SAEs) to map layer-wise task vectors into sparse feature codes, where task-specific components can be separated more explicitly~\cite{bricken2023monosemanticity,cunningham2023sparse}. The task vectors are then fused at the feature level and mapped back to the parameter space. Since applying SAE-based disentanglement to all layers can be computationally expensive, inspired by zeroth-order optimization techniques that estimate search directions from function evaluations~\cite{nesterov2017random,malladi2023fine}, we further introduce a lightweight Group-Ranked Zeroth-Order Optimizer (GR-ZOO) to identify task-critical layers. By applying SAE-based merging only to these selected layers, our method reduces computational overhead while preserving the benefits of feature-level disentanglement.

The main contributions of this paper are summarized as follows:
\begin{itemize}
    \item We analyze task interference in model merging from the perspective of \textit{superposition}, suggesting that highly entangled task vectors can be difficult to separate within the original parameter space.
    \item We propose \textbf{High-Dimensional Sparse Disentanglement Merging}, which uses SAEs to project task vectors into a high-dimensional sparse feature space and performs feature-level fusion to reduce interference.
    \item We introduce a Group-Ranked Zeroth-Order Optimizer (GR-ZOO), a forward-pass-only layer selection method that identifies task-critical layers and avoids applying SAE-based disentanglement to every layer.

\end{itemize}

\section{Related Work}
\label{sec:related_work}

\subsection{Multi-Task Model Merging}
\label{subsec:multi_task_merging}

Model merging aims to integrate multiple task-specific expert models into a single generalist model without joint retraining. Recent advancements fall into two primary trajectories: parameter arithmetic and parameter decomposition.

\paragraph{Parameter Arithmetic.}
This line of work operates directly within the original parameter space. Approaches like Model Soups~\cite{wortsman2022model} and Task Arithmetic~\cite{ilharco2022editing} merge weights via linear interpolation or vector addition. To mitigate severe task interference, subsequent methods introduce sparsification and conflict resolution mechanisms, such as trimming and sign election (TIES-Merging~\cite{yadav2023ties}), randomized dropping (DARE~\cite{yu2024language}), and magnitude-aware sampling (DELLA~\cite{deep2024della}). 

\paragraph{Parameter Decomposition.}
To overcome low-dimensional limitations, a complementary trajectory projects task updates into structured subspaces prior to fusion. Methods like KnOTS~\cite{stoica2024knots} and TSV-M~\cite{gargiulo2025task} align or compress updates via Singular Value Decomposition (SVD). Other frameworks introduce spectral truncation (STAR~\cite{lee2025star}), adaptive rank selection (AdaRank~\cite{lee2026adarank}), or explicit subspace partitioning (ISO~\cite{marczak2025no}, WIDEN~\cite{xiong2024widen}, ESM~\cite{li2026model}). 

Although highly scalable, these coordinate-wise heuristic manipulations operate directly in the original parameter space and may struggle when
multiple task updates exhibit substantial interference. Despite providing stronger functional decoupling, these strategies predominantly rely on linear composability assumptions or orthogonal transformations. When functionally distinct capabilities are trapped in non-orthogonal, superposed structures, such rigid linear approximations fail to fully isolate useful task directions from overlapping, conflicting components.

\subsection{Superposition and Sparse Feature Disentanglement}
\label{subsec:superposition_disentanglement}

Mechanistic interpretability establishes that neural networks tend to encode latent features far exceeding their available dimensions into overcomplete, non-orthogonal hidden directions---a phenomenon known as \textit{superposition}~\cite{elhage2022toy, scherlis2022polysemanticity}. This causes \textit{polysemanticity}, where individual parameter coordinates activate across multiple orthogonal contexts~\cite{henighan2023superposition}. Consequently, task vectors are highly superposed mixtures of multi-semantic components. Attempting to resolve interference via low-dimensional coordinate pruning or linear orthogonalization inherently creates a zero-sum trade-off, distorting crucial task capabilities.

To resolve superposition, Sparse Autoencoders (SAEs) have emerged as powerful tools for unsupervised dictionary learning, untangling features into monosemantic components across foundation models~\cite{yun2021transformer, cunningham2023sparse, bricken2023monosemanticity, templeton2024scaling, gao2025scaling}. However,~\cite{cui2026limits} demonstrate that vanilla SAEs suffer from magnitude shrinkage induced by $L_1$ penalties, highlighting the need for rigid activation gating and geometric regularizations. 

While prior work mainly applies Sparse Autoencoders (SAEs) to activation analysis, we extend sparse feature disentanglement to task-vector merging by projecting parameter updates into a high-dimensional sparse latent space before fusion, reducing interference from low-dimensional parameter overlap.

\section{Methodology}

In this section, we present the proposed \textbf{High-Dimensional Sparse Disentanglement Merging} framework. We begin by formulating the model merging problem and theoretically show that sparse decomposition of task vectors can reduce cross-task interference under feature superposition. We then employ an improved \textbf{Sparse Autoencoder} (SAE) to project layer-wise task vectors into a high-dimensional sparse feature space, where task-relevant components can be represented more explicitly. To make sparse disentanglement computationally feasible, we introduce \textbf{Group-Ranked Zeroth-Order Optimizer} (GR-ZOO) to identify task-critical layers, so that the subsequent feature-level processing can be restricted to a small subset of layers. Finally, we design a feature-aware fusion strategy in the sparse latent space and decode the fused representation back to the parameter space, enabling effective multi-task model merging while reducing task interference.

\subsection{Problem Setup}

\noindent\textbf{Model Merging.}
Let $f_{\theta_0}$ denote a pre-trained base model with parameters 
$\theta_0 \in \mathbb{R}^{d}$. Given $N$ downstream tasks 
$\{\mathcal{T}_i\}_{i=1}^{N}$, each task-specific expert model 
$f_{\theta_i}$ is obtained by fine-tuning the same base model on task 
$\mathcal{T}_i$, where $\theta_i \in \mathbb{R}^{d}$ denotes the parameters 
of the $i$-th expert. The goal of model merging is to construct a single 
model $f_{\theta_m}$ that preserves the capabilities of all expert models 
without accessing their original training data or performing additional 
full fine-tuning. Formally, model merging can be written as
\begin{equation}
    \theta_m = \mathcal{A}(\theta_0, \theta_1, \ldots, \theta_N),
\end{equation}
where $\mathcal{A}(\cdot)$ denotes a merging algorithm.

\noindent\textbf{Task Vector.}
Following task-vector-based model merging~\cite{ilharco2022editing,yu2024language}, we define the task vector of expert 
$i$ as the parameter displacement from the base model:
\begin{equation}
    \tau_i = \theta_i - \theta_0 .
\end{equation}
The task vector $\tau_i$ represents the task-specific knowledge acquired by 
fine-tuning on task $\mathcal{T}_i$. A general task-vector merging process can 
then be formulated as
\begin{equation}
    \theta_m = \theta_0 + \mathcal{M}(\tau_1, \tau_2, \ldots, \tau_N),
\end{equation}
where $\mathcal{M}(\cdot)$ is a task-vector fusion operator. A typical example 
is weighted linear merging~\cite{ilharco2022editing}:
\begin{equation}
    \theta_m = \theta_0 + \Delta,
    \quad
    \Delta = \sum_{i=1}^{N} \alpha_i \tau_i ,
\end{equation}
where $\alpha_i$ controls the contribution of expert $i$.

Since modern language models consist of multiple parameterized layers, we can 
also define the layer-wise task vector as
\begin{equation}
    \tau_i^{(\ell)} = \theta_i^{(\ell)} - \theta_0^{(\ell)},
    \quad \ell = 1,2,\ldots,L ,
\end{equation}
where $\theta_i^{(\ell)}$ denotes the parameters of the $\ell$-th layer of 
expert $i$. This layer-wise view is important because task-specific knowledge 
and cross-task conflicts may be unevenly distributed across layers.

\paragraph{Superposition.}
Superposition~\cite{elhage2022toy,scherlis2022polysemanticity} suggests 
that large neural networks may encode more latent features than the number of available independent dimensions, causing multiple features to share the same representational directions. We view task interference through the lens of feature superposition: task vectors may contain sparse combinations of latent capability directions that are not well separated in the original parameter space. As a result, coordinate-wise pruning or orthogonal decomposition may not fully remove cross-task interference without distorting useful capabilities. We formalize this intuition and provide theoretical bounds in Appendix~\ref{app:proofs}. These results motivate projecting
task vectors into a high-dimensional sparse feature space before fusion.

\subsection{High-Dimensional Sparse Disentanglement via Improved SAEs}

To reduce task interference induced by feature superposition in the original parameter space, we project layer-wise task vectors into a high-dimensional sparse latent space. Specifically, for a selected task vector $\tau_i^{(\ell)}\in \mathbb{R}^d$. We train a Sparse Autoencoder~\cite{cunningham2023sparse} (SAE) to encode $\tau_i^{(\ell)}$ into an overcomplete sparse representation and then reconstruct it back to the parameter space:
\begin{equation}
    h = E(\tau_i^{(\ell)})\Rightarrow 
    z = \mathrm{TopK}(h)\Rightarrow  
    \hat{\tau}_i^{(\ell)} = D(z)
\end{equation}
where $E(\cdot)$ and $D(\cdot)$ denote the encoder and decoder, respectively. The latent dimension is set much larger than the input dimension, enabling task vectors to be represented by a small number of activated latent features.

Although standard SAEs are commonly trained with an $L_1$ sparsity penalty, directly applying them to model merging is suboptimal. The $L_1$ penalty tends to shrink latent activation magnitudes, which may distort the scale of task-specific updates~\cite{rajamanoharan2024improving}. Moreover, standard SAE training often suffers from inactive latent units, commonly known as dead neurons, leading to an inefficient use of the high-dimensional latent space. Finally, without proper constraints on the decoder, the model may exploit scale ambiguity between latent activations and decoder weights, resulting in unstable or collapsed feature representations. To address these issues, we introduce the following modifications.

\paragraph{Top-$K$ sparse activation.}
Instead of imposing an $L_1$ penalty on the latent codes, we enforce sparsity using a hard Top-$K$ activation:
\begin{equation}
    z_j =
    \begin{cases}
        h_j, & j \in \mathcal{S}_K(h), \\
        0, & \text{otherwise},
    \end{cases}
\end{equation}
where $\mathcal{S}_K(h)$ denotes the indices of the $K$ largest activations in $h$. This operation preserves the magnitude of salient latent features while ensuring that only a limited number of features are active for each input. Compared with soft sparsity induced by $L_1$ regularization, Top-$K$ sparsity avoids systematic activation shrinkage and encourages different latent units to compete for representing task-relevant directions~\cite{gao2025scaling}.

\paragraph{Residual fitting for inactive features.}
To improve the utilization of the overcomplete latent space, we introduce an auxiliary residual fitting objective for inactive features. Given the reconstruction residual
\begin{equation}
    r = \tau_i^{(\ell)} - \hat{\tau}_i^{(\ell)},
\end{equation}
we encourage rarely activated latent units to explain the residual directions that are not captured by the current active features. Let $\mathcal{D}$ denote the set of inactive or rarely used latent units, and let $z^{res}$ be the auxiliary sparse code obtained by applying Top-$K$ selection only within $\mathcal{D}$. The residual fitting loss is defined as
\begin{equation}
    \mathcal{L}_{res}
    =
    \left\|
        \mathrm{sg}(r) - D(z^{res})
    \right\|_2^2,
\end{equation}
where $\mathrm{sg}(\cdot)$ denotes the stop-gradient operation. This auxiliary loss does not alter the main reconstruction target, but provides additional learning signals for under-utilized features, encouraging them to explore directions not sufficiently represented by the active latent subspace.

\paragraph{Decoder normalization and orthogonality regularization.}
To avoid scale ambiguity between latent activations and decoder weights, we normalize each decoder atom after every update:
\begin{equation}
    w_j \leftarrow \frac{w_j}{\|w_j\|_2},
\end{equation}
where $w_j$ is the $j$-th column of the decoder matrix $W_{dec}$. This constraint ensures that feature competition is primarily determined by directional similarity rather than weight magnitude. In addition, we introduce an orthogonality regularization term on the normalized decoder atoms:
\begin{equation}
    \mathcal{L}_{ortho}
    =
    \left\|
        W_{dec}^{\top} W_{dec} - I
    \right\|_F^2.
\end{equation}
This penalty discourages redundant dictionary atoms and promotes a more diverse latent feature space, which is beneficial for disentangling task-specific components.

The final training objective of the improved SAE is
\begin{align}
    \mathcal{L}_{SAE}
    &=
    \underbrace{\|\omega - D(\mathrm{TopK}(E(\tau_i^{(\ell)})))\|_2^2}_{\mathcal{L}_{rec}}
    \notag\\
    &+\lambda_{res}\mathcal{L}_{res}\notag\\
    &+\lambda_{ortho}\mathcal{L}_{ortho}
\end{align}
where $\lambda_{res}$ and $\lambda_{ortho}$ control the strengths of residual fitting and orthogonality regularization, respectively. After training, each selected layer-wise task vector is encoded into the sparse latent space, merged at the feature level, and decoded back to the original parameter space for model merging.

\subsection{Group-Ranked Zeroth-Order Optimizer (GR-ZOO)}

Although high-dimensional sparse disentanglement can effectively reduce task interference, applying SAE-based mapping to every layer of a large language model is computationally expensive. Moreover, task interference is usually not uniformly distributed across all layers. Therefore, instead of performing sparse disentanglement on all layers, it is more efficient to first identify a small subset of task-critical layers and only apply SAE-based merging to these layers.

A natural way to measure the importance of a layer is to estimate how sensitively the task loss changes when the parameters of that layer are perturbed. However, directly computing gradients for all candidate layers requires backpropagation through the full model, which is memory-intensive and inefficient for large-scale model merging. Zeroth-order optimization~\cite{malladi2023fine,zhang2024revisiting} provides a lightweight alternative by estimating local sensitivity using only forward evaluations. Given a model parameter $\theta$ and a loss function $\mathcal{L}(\theta)$, a standard two-point zeroth-order estimator samples a random direction $u$ and estimates the directional derivative as
\begin{equation}
    g(u)
    =
    \frac{
        \mathcal{L}(\theta + \epsilon u)
        -
        \mathcal{L}(\theta - \epsilon u)
    }{2\epsilon},
\end{equation}
where $\epsilon$ is the perturbation scale. This avoids explicit gradient computation and is thus suitable for layer-wise sensitivity estimation.

Nevertheless, directly using standard zeroth-order estimates for layer selection is unstable in the context of LLMs~\cite{Yu_2025_ICCV,park2024mezoadam,gautam2024variancereduced}. The loss landscape of large models is highly non-smooth and anisotropic, making single-direction estimates extremely noisy. This issue makes naive ZOO prone to high-variance estimates and unstable selection results.

To address these limitations, we propose a \textbf{Group-Ranked Zeroth-Order Optimizer} (GR-ZOO). We partition the model parameters into $G$ groups, where each group corresponds to a layer or a transformer block. Let $\theta^{ref}$ denote a reference expert model, and let $P_g$ be the parameter mask that only perturbs the $g$-th group. For each task $\mathcal{T}_i$, we evaluate the task loss on a small calibration set $\mathcal{B}_i$. For each group $g$, we sample $R$ normalized random directions $\{u_{g,r}\}_{r=1}^{R}$ and compute the symmetric zeroth-order response:
\begin{equation}
\begin{aligned}
    s_{i,g,r}
    &=
    \left|
    \frac{
    \ell_i(\theta^{ref}+\delta_{g,r})
    -
    \ell_i(\theta^{ref}-\delta_{g,r})
    }{2\epsilon}
    \right|, \\
    \ell_i(\theta)
    &\triangleq
    \mathcal{L}_i(\theta; \mathcal{B}_i),
    \qquad
    \delta_{g,r}
    \triangleq
    \epsilon P_g u_{g,r}.
\end{aligned}
\end{equation}
A larger value of $s_{i,g,r}$ indicates that the task loss is more sensitive to perturbations in group $g$, suggesting that this group is more relevant to task $\mathcal{T}_i$.

Instead of directly averaging the raw sensitivity values, GR-ZOO converts them into rank scores. For each task $\mathcal{T}_i$ and perturbation round $r$, we rank all groups according to their zeroth-order responses $\{s_{i,g,r}\}_{g=1}^{G}$. Let $\mathrm{rank}_{i,r}(g)$ denote the rank of group $g$, where a smaller rank indicates stronger sensitivity. The normalized rank score is defined as
\begin{equation}
    q_{i,g,r}
    =
    \frac{G - \mathrm{rank}_{i,r}(g) + 1}{G}.
\end{equation}
The final importance score of group $g$ is obtained by aggregating its rank scores across tasks and perturbation rounds:
\begin{equation}
    S_g
    =
    \frac{1}{NR}
    \sum_{i=1}^{N}
    \sum_{r=1}^{R}
    q_{i,g,r}.
\end{equation}
This rank-based aggregation suppresses the influence of unstable raw loss magnitudes and makes the importance scores more comparable across layers and tasks. Meanwhile, using multiple perturbation directions reduces the variance of zeroth-order estimation and yields a more reliable layer ranking.

After computing $\{S_g\}_{g=1}^{G}$, we select the top-$M$ groups with the highest importance scores:
\begin{equation}
    \mathcal{G}^{*}
    =
    \mathrm{TopM}\left(\{S_g\}_{g=1}^{G}\right).
\end{equation}
SAE-based high-dimensional sparse disentanglement is then applied only to the selected groups $\mathcal{G}^{*}$, while the remaining layers are merged using a lightweight parameter-space strategy. In this way, GR-ZOO serves as an efficient layer selection mechanism that concentrates the computational budget on the most task-sensitive parts of the model, thereby preserving the benefit of sparse feature disentanglement while substantially reducing the overall merging cost.

\subsection{Differentiated Parameter Fusion Strategy}

After projecting the task vectors of critical layers into the high-dimensional sparse feature space, a naive additive fusion rule treats all latent features uniformly. However, some features may correspond to common capabilities shared across multiple tasks. Directly accumulating such overlapping features can lead to norm inflation after decoding, thereby destabilizing the merged parameters. To address this issue, we distinguish shared features from task-specific features and apply different fusion operations to them.

Let $\boldsymbol{\mu}_i$ and $\boldsymbol{\nu}_i$ denote the $i$-th latent feature components of two task vectors in the sparse feature space. We measure their similarity by cosine similarity:
\begin{equation}
    s_i =
    \frac{
        \langle \boldsymbol{\mu}_i, \boldsymbol{\nu}_i \rangle
    }{
        \|\boldsymbol{\mu}_i\|_2 \|\boldsymbol{\nu}_i\|_2 + \epsilon
    },
\end{equation}
where $\epsilon$ is a small constant for numerical stability. Given a similarity threshold $\tau$, we partition the latent features into a shared set and a unique set:
\begin{equation}
    \mathcal{S} = \{ i \mid s_i \geq \tau \},
    \qquad
    \mathcal{U} = \{ i \mid s_i < \tau \}.
\end{equation}
The fused latent feature $\boldsymbol{\omega}^{\mathrm{merged}}_i$ is then computed as
\begin{equation}
    \boldsymbol{\omega}^{\mathrm{merged}}_i =
    \begin{cases}
        \frac{1}{2}\left(\boldsymbol{\mu}_i + \boldsymbol{\nu}_i\right),
        & i \in \mathcal{S}, \\[4pt]
        \boldsymbol{\mu}_i + \boldsymbol{\nu}_i,
        & i \in \mathcal{U}.
    \end{cases}
\end{equation}

For features in the shared set $\mathcal{S}$, high similarity indicates that the corresponding latent directions encode overlapping capabilities across tasks. Therefore, we use mean fusion to preserve these common features while preventing repeated amplification of their magnitudes. In contrast, features in the unique set $\mathcal{U}$ are treated as task-specific components. 

Finally, the fused latent representation is decoded back to the parameter space of the selected critical layers. In this way, the proposed strategy avoids over-amplifying shared capabilities while retaining task-specific features, leading to a more stable and less interfering merged model.

\section{Experiments}

\begin{table*}[t]
\centering
\small
\renewcommand{\arraystretch}{1.2}
\begin{tabular}{l|c|c|c|c}
\hline
\textbf{Methods} 
& \textbf{Math (GSM8k)} 
& \textbf{Code (HumanEval+)} 
& \textbf{IF (IFEval)} 
& \textbf{Average} \\
\hline

Task Arithmetic 
& 83.85 
& 73.80 
& 45.47 
& 67.71 \\

TIES-Merge 
& 83.40 
& 75.00 
& 38.45 
& 65.62 \\

DARE+TIES
& 83.62 
& 72.60 
& 36.78 
& 64.33 \\

Fisher-Merge 
& 84.00
& 75.60
& 41.59 
& 67.06 \\

TSV-Merge 
& 83.32 
& 72.88 
& 45.77
& 67.32 \\

WUDI-Merging 
& 82.79 
& 73.49 
& 44.74 
& 67.01 \\

EMR-Merging 
& 83.70 
& 74.71 
& 43.08 
& 67.16 \\

DELLA 
& 83.40 
& 75.32
& 42.89
& 67.20 \\

\hline

\textbf{Ours}
& 85.22
& 74.40
& 45.84
& 68.49
\\

\hline
\end{tabular}

\caption{
Main multi-task merging results on Qwen2.5-7B.
}
\label{tab:main_results}
\end{table*}

\begin{table*}[t]
\centering
\small
\renewcommand{\arraystretch}{1.2}
\begin{tabular}{l|c|cccc|c|c|c}
\hline
\multirow{2}{*}{\textbf{Methods}}
& \textbf{Math}
& \multicolumn{4}{c|}{\textbf{General}}
& \textbf{Code}
& \textbf{Safe}
& \multirow{2}{*}{\textbf{Average}} \\
\cline{2-8}

& GSM8k
& STEM
& Social Sciences
& Humanities
& Others
& HumanEval
& BeaverTail
& \\

\hline

Task Arithmetic
& 33.74
& 8.91
& 9.39
& 7.48
& 7.00
& 0.00
& 41.60
& 15.45 \\

TIES-Merge
& 51.10
& 19.48
& 15.70
& 14.73
& 11.75
& 32.93
& 61.05
& 29.53 \\

DARE+TIES
& 52.16
& 20.44
& 14.85
& 13.62
& 11.14
& 31.71
& 49.91
& 27.69 \\

Fisher-Merge
& 60.80
& 22.13
& 6.47
& 10.74
& 5.74
& 36.59
& 39.00
& 25.92 \\

TSV-Merge
& 52.00
& 30.10
& 33.30
& 29.70
& 38.80
& 15.00
& 29.00
& 32.24 \\

WUDI-Merging
& 36.00
& 24.30
& 22.60
& 28.90
& 25.90
& 8.00
& 12.00
& 22.50 \\

EMR-Merging
& 33.00
& 26.20
& 23.80
& 25.00
& 23.50
& 8.00
& 3.00
& 20.40 \\

DELLA
& 54.00
& 30.10
& 31.00
& 28.90
& 32.90
& 37.00
& 22.00
& 33.70 \\

\hline

\textbf{Ours}
& 51.55
& 29.42
& 36.53
& 26.65
& 28.35
& 32.32
& 50.55
& 36.48 \\

\hline
\end{tabular}

\caption{
Performance scaling to 4 tasks on Qwen2.5-1.5B.
Average denotes the unweighted arithmetic mean over the seven reported
evaluation metrics: GSM8K, four General capability domains, HumanEval,
and BeaverTail. The same aggregation rule is applied consistently to all
methods. 
}
\label{tab:scaling}
\end{table*}

To comprehensively evaluate the proposed High-Dimensional Sparse Disentanglement Merging framework, we conduct extensive empirical validations across diverse natural language processing tasks and model scales. In this section, we first detail the experimental setup, followed by an in-depth analysis of critical layer identification, main multi-task merging results, performance scalability under increasing task conflict, and comprehensive ablation studies.

\subsection{Experimental Setup}
\paragraph{Models and Task Selection.} We select the Qwen2.5~\cite{qwen2025qwen25technicalreport} series as the foundation models and conduct experiments at two levels:
\begin{itemize}
    \item \textbf{Main Results (7B scale):} We utilize Qwen2.5-7B to verify merging capabilities at a large parameter scale. Three expert models are fine-tuned: Mathematical Reasoning~\cite{gsm8k} (GSM8k), Code Generation~\cite{humanevalplus} (HumanEval+), and Instruction Following~\cite{ifeval} (IFEval).
    \item \textbf{Analysis (1.5B scale):} We utilize Qwen2.5-1.5B for high-conflict scenarios and mechanistic analysis. Four experts are constructed: Math (GSM8k), Code~\cite{humaneval} (HumanEval), General Knowledge~\cite{mmlu} (MMLU subsets), and Safety Alignment~\cite{beavertail} (BeaverTail).
\end{itemize}

\paragraph{Baselines.} 
We compare our framework against representative training-free model
merging methods, including {Task Arithmetic} \citep{ilharco2022editing},
{TIES-Merge} \citep{yadav2023ties}, {DARE} \citep{yu2024language},
{Fisher-Merge}\citep{matena2022merging}, {TSV-Merge} \citep{gargiulo2025task},
{WUDI-Merging} \citep{cheng2025whoever},
{EMR-Merging} \citep{huang2024emr},
and {DELLA-Merging} \citep{deep2024della}.

\begin{figure}[h]
  \centering
  \includegraphics[width=\linewidth]{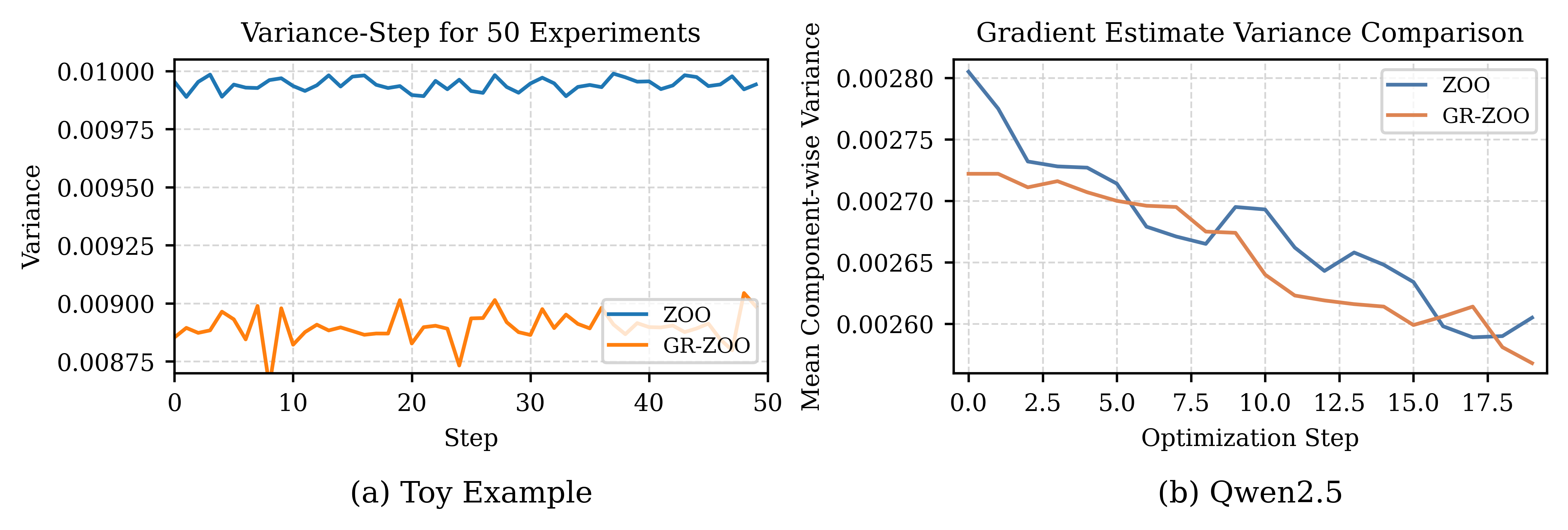}
  \caption{Gradient estimate variance comparison: GR-ZOO effectively smooths the variance compared to standard ZOO.}
  \label{fig:grzoo_variance}
\end{figure}

\subsection{Effectiveness of Critical Layer Identification}
We validate GR-ZOO against full-gradient and random layer selection on Qwen2.5-1.5B. As shown in \textbf{Figure~\ref{fig:grzoo_variance}}, standard ZOO~\cite{malladi2023fine} exhibits high estimation variance in complex LLM loss landscapes, whereas GR-ZOO substantially stabilizes the estimates through group ranking.

As illustrated in \textbf{Figure~\ref{fig:layer_importance}}, the layers selected by GR-ZOO achieve recovery rates of 62.31\% for Math and 41.46\% for Code, closely approaching the full-gradient upper bounds of 65.73\% and 46.34\%, respectively, while significantly outperforming random selection. These results demonstrate that GR-ZOO can effectively identify task-critical layers using only forward-pass evaluations.

\begin{figure}[h]
  \centering
  \includegraphics[width=0.9\linewidth]{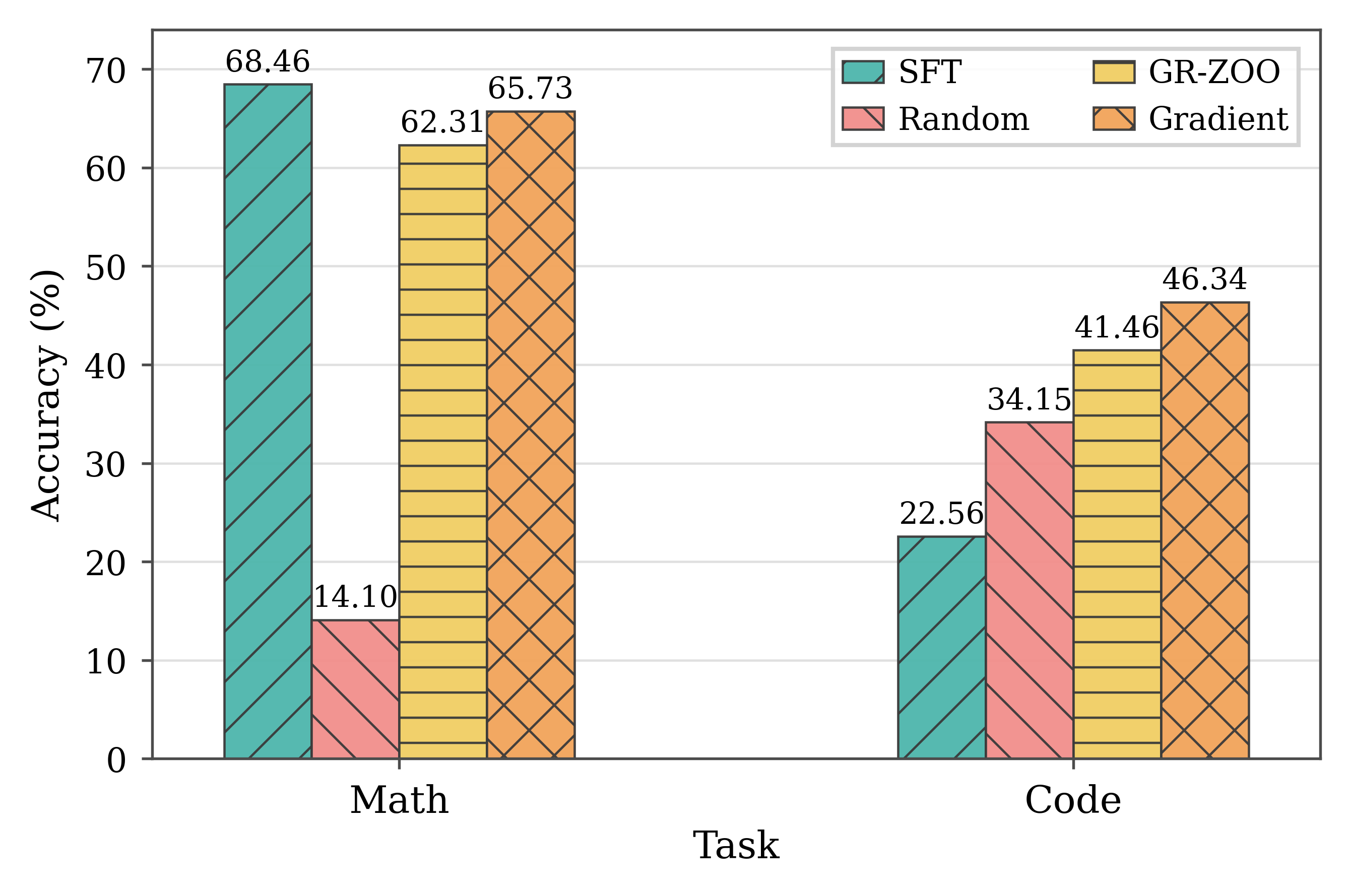}
  \caption{Performance comparison of different layer selection strategies. GR-ZOO closely approximates the theoretical upper bound (Gradient) while significantly outperforming the Random baseline.}
  \label{fig:layer_importance}
\end{figure}

\subsection{Main Multi-Task Merging Results}
We evaluate 3-task merging (Math, Code, IF) on Qwen2.5-7B, as shown in \textbf{Table~\ref{tab:main_results}}. Conventional arithmetic methods suffer from severe performance drops, particularly on sensitive tasks like IFEval (dropping to 38.45 and 36.78). This indicates that low-dimensional masking or dropping destroys shared instruction features. 

Our framework circumvents this by projecting overlapping features into an orthogonal high-dimensional space. It achieves state-of-the-art performance with an \textbf{Average score of 68.49}, notably outperforming the best baseline Task Arithmetic (67.71) and ensuring near-lossless retention of IF and Math capabilities.

\subsection{Performance Degradation under Task Scaling}
The bottleneck of model merging is exposed as the number of tasks increases and conflicts intensify. We extend the Qwen2.5-1.5B setup from 3 tasks to 4 tasks by adding Safety Alignment, with the detailed results presented in \textbf{Table~\ref{tab:scaling}}. 

In the high-conflict 4-task scenario, baselines degrade catastrophically: Task Arithmetic's average score drops to 15.45, and TIES-Merge falls to 29.53. Our framework exhibits superior robustness, maintaining an average score of 36.48. This represents a \textbf{6.95\% improvement over TIES-Merge}, effectively preserving general and math capabilities while maintaining strict safety alignment comparable to the expert model.

\subsection{Ablation Study}

We conduct controlled ablations on Qwen2.5-1.5B to examine the effects of
the projection space, layer-selection strategy, and individual SAE
components. Detailed per-task results and representation-level diagnostics
are provided in Appendix~\ref{app:extended_results}.

\paragraph{Ablation Results.}
Under matched settings, SAE and GR-ZOO achieve the highest average scores
among the projection and layer-selection alternatives, respectively
(Table~\ref{tab:projection_selection_summary}). Our complete SAE design
also improves the average from 39.30 to 40.18 in the 3-task setting and
from 35.24 to 36.48 in the 4-task setting. Detailed results are provided
in Appendix~\ref{app:extended_results}.

\begin{table}[H]
\centering
\small
\setlength{\tabcolsep}{5pt}
\renewcommand{\arraystretch}{1.1}

\begin{tabular}{llc}
\toprule
\textbf{Component}
& \textbf{Variant}
& \textbf{Metric Avg.} \\
\midrule

\multirow{3}{*}{Projection Space}
& PCA
& 30.80 \\

& Random Projection
& 31.01 \\

& SAE
& \textbf{31.56} \\

\midrule

\multirow{4}{*}{Layer Selection}
& First Layers
& 28.27 \\

& Random Selection
& 27.64 \\

& Fisher Selection
& 29.90 \\

& GR-ZOO
& \textbf{31.56} \\

\bottomrule
\end{tabular}

\caption{
Average performance of different projection spaces and layer-selection
strategies. 
}
\label{tab:projection_selection_summary}
\end{table}

\section{Conclusion}

In this paper, we propose the High-Dimensional Sparse Disentanglement Merging framework to address the parameter superposition bottleneck in multi-task merging. By utilizing an improved Top-$K$ Sparse Autoencoder, our method projects overlapping updates into an orthogonal high-dimensional space, substantially eliminating destructive interference. Additionally, we introduce a Group-Ranked Zeroth-Order Optimizer (GR-ZOO) to efficiently identify task-critical layers for targeted disentanglement. Extensive evaluations demonstrate that our approach achieves state-of-the-art performance and exhibits exceptional robustness in highly conflicting scenarios.

\section*{Limitations}
While the proposed framework significantly mitigates task interference, it possesses certain limitations. First, despite the efficiency gains introduced by GR-ZOO, training the Sparse Autoencoder on the critical layers still incurs additional computational overhead prior to merging, compared to purely arithmetic methods (e.g., Task Arithmetic) which operate instantaneously. Second, the framework introduces hyper-parameters, notably the high-dimensional expansion factor and the cosine similarity threshold ($\tau_{sim}$). Although our empirical defaults perform robustly across Qwen2.5 models, optimal thresholds may vary slightly across different model architectures (e.g., LLaMA-3~\cite{grattafiori2024llama} or Mistral~\cite{liu2026ministral}) or significantly different task scales, requiring minor calibration. Future work will explore dynamic, parameter-free thresholding mechanisms.

\section*{Ethics Statement}
This work relies exclusively on publicly available evaluation datasets. We neither collect new human-subject data nor attempt to identify individuals from the datasets. These data are used solely for model training and aggregate evaluation, and we report benchmark-level metrics rather than reproducing potentially harmful examples. All external datasets, benchmarks, models, and judges are properly attributed to their original creators and used in accordance with their intended research purposes. Furthermore, we acknowledge that large language models are prone to hallucinations—generating plausible but factually incorrect or fabricated content—and therefore emphasize the necessity of careful human review and verification of all model outputs prior to deployment or reliance in real-world applications.
\bibliography{custom}

\clearpage
\twocolumn
\appendix
\begin{center}
    \LARGE \textbf{Supplementary Material}
\end{center}
\vspace{1cm}

\section{Model Configuration}
\label{app:model_config}

This section describes the base model and the training details of the task-specific experts discussed in our main experiments. All expert models are initialized from the same foundation model and undergo full-parameter supervised fine-tuning (SFT). To ensure training stability and efficiency, we train all experts using bfloat16 precision, DeepSpeed ZeRO-3, gradient checkpointing, and a cosine learning rate schedule.

\begin{itemize}
    \item \textbf{Base Model (Qwen2.5-7B):} We adopt the pre-trained Qwen2.5-7B model as our foundation base model for all downstream expert tuning and model merging experiments.
    
    \item \textbf{Math Expert:} We fine-tune the base model using approximately 7,160 high-quality mathematical reasoning samples. The data is constructed via rejection sampling (RS): a teacher model generates multiple candidate responses featuring step-by-step reasoning and a final answer. These candidates are then strictly filtered based on formatting and mathematical validity.
    
    \item \textbf{Code Expert:} We train the code expert on 8,646 MBPP-style Python programming tasks. For each instruction (comprising a problem description, function requirements, and test cases), teacher-generated code snippets are extracted and executed against unit and challenge tests. Only test-passing, syntactically correct, and non-duplicate solutions are retained for SFT. 
    
    \item \textbf{Instruction Following (IF) Expert:} We construct 5,473 samples focusing on automatically verifiable rule constraints, such as JSON formatting, numbered lists, word limits, and exact keyword inclusions. Teacher-generated responses are validated through deterministic rule checkers, ensuring that only strictly compliant samples are added to the training set, thereby eliminating manual annotation costs. 
\end{itemize}

\section{Visualization of Sparse Feature Selection}
\label{app:sparse_feature_visualization}

We visualize the element-wise weight differences between our SAE-based merge and Task Arithmetic (TA). Specifically, for selected parameter indices, we compute
$\Delta w = w_{\mathrm{Ours}} - w_{\mathrm{TA}}$,
where $w_{\mathrm{Ours}}$ and $w_{\mathrm{TA}}$ denote the merged weights produced by our method and TA, respectively. The resulting distributions are shown in Figure~\ref{fig:param_diff}.

The near-zero-centered envelopes suggest that our SAE-based merge largely preserves the global structure of TA instead of introducing broad dense shifts. Meanwhile, the discontinuous spike-like patterns are consistent with sparse feature selection, indicating that the SAE mainly calibrates selected parameter directions rather than uniformly rescaling all updates.

\begin{figure}[t!]
  \centering
  \includegraphics[width=0.88\linewidth]{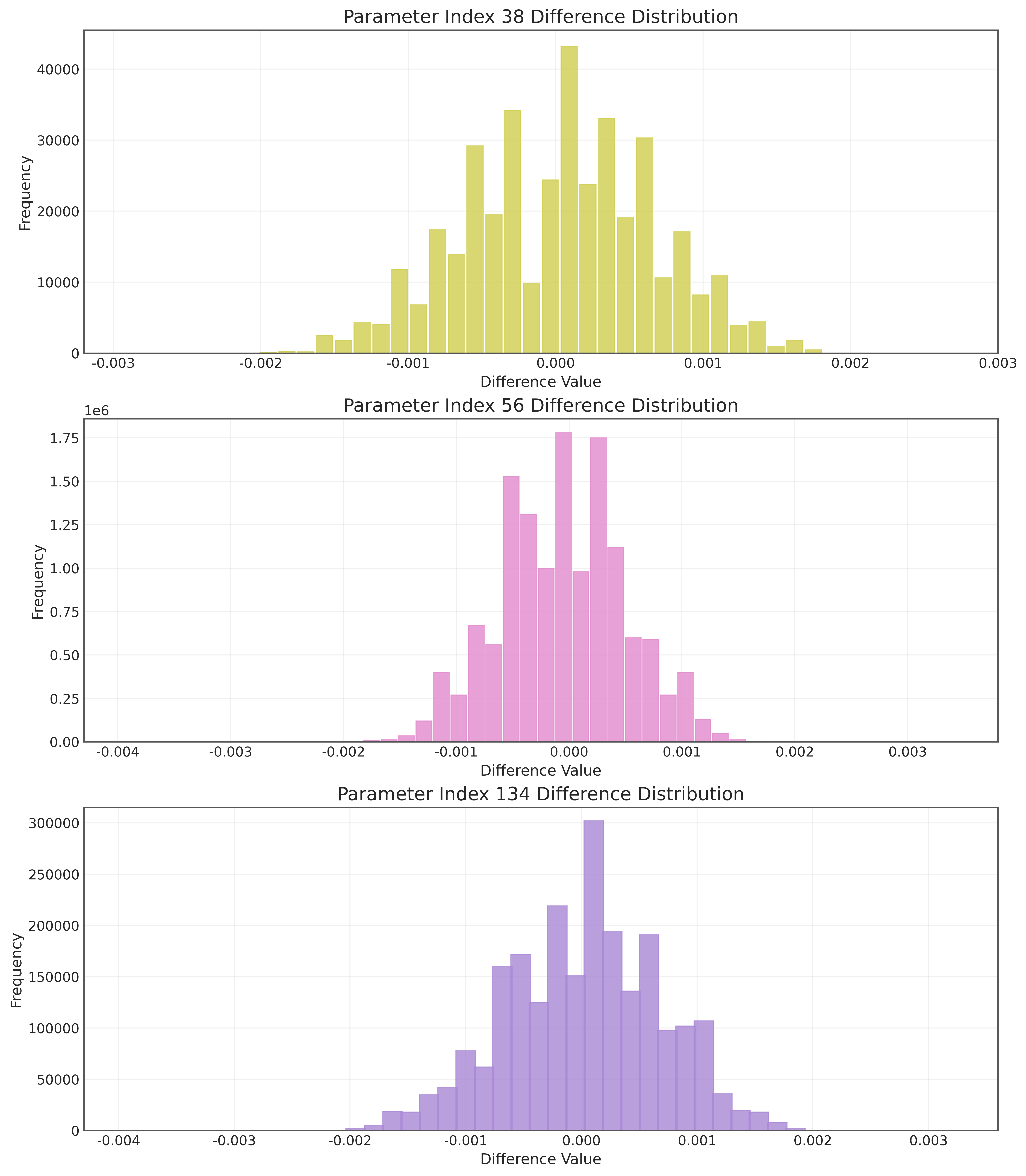}
  \caption{Distribution of element-wise parameter differences
  $\Delta w = w_{\mathrm{Ours}} - w_{\mathrm{TA}}$
  across selected parameter indices. The distributions are centered near zero, while their spike-like structures suggest sparse and selective parameter calibration.}
  \label{fig:param_diff}
\end{figure}

\section{Computational Cost and Training Details}
\label{app:computational_cost}

We profile the computational cost of the GR-ZOO and SAE pipeline and
summarize the training configuration in Table~\ref{tab:computational_cost}.
For SAE training, we report the cost of processing one selected layer.

\begin{table}[t]
\centering
\scriptsize
\setlength{\tabcolsep}{2pt}
\renewcommand{\arraystretch}{0.95}

\begin{tabular}{@{}p{0.50\linewidth}|c@{}}
\hline
\textbf{Item} & \textbf{Setting} \\
\hline
Candidate / selected layers & 197 / 16 \\
Fraction of layers receiving SAE processing & 8.12\% \\
Number of trained SAEs & 16 \\
Cold GR-ZOO time & 467.13s / 7.79min \\
Single-layer SAE training and latent encoding 
& 366.75s / 6.11min \\
Peak GPU memory & 25,153.93 MB \\
Sparse activation & Hard Top-$K$ \\
Default Top-$K$ & $K=32$ \\
Post-Top-$K$ active ratio & 8.33\% \\
Similarity threshold & $\tau=10^{-4}$ \\
Objective & $L_{\mathrm{SAE}}$ \\
Optimizer & AdamW \\
Learning rate & $10^{-4}$ \\
Batch size & 512 \\
Training epochs & 10 \\
Latent expansion factor & $4\times$ \\
\hline
\end{tabular}

\caption{
Computational cost and training configuration of GR-ZOO and SAE.
SAE cost is measured for one selected layer.
}
\label{tab:computational_cost}
\end{table}

All experiments are conducted on NVIDIA A800 GPUs with 80GB memory.
The expert models are trained with bfloat16 precision, DeepSpeed ZeRO-3,
gradient checkpointing, and a cosine learning rate schedule.

GR-ZOO introduces only an offline layer-selection cost. After selecting
task-critical layers, SAE processing is applied only to these layers.
The decoded merged model preserves the original architecture and introduces
no additional inference overhead.

\section{Disentanglement Analysis}
\label{app:disentanglement}

To further verify whether SAE provides effective feature disentanglement,
we compare the activation patterns of raw task vectors and SAE latent
representations. We measure feature overlap using pairwise active overlap
and active Jaccard similarity, where lower values indicate less shared
activation. We also report the ratio of task-specific and shared features
to characterize the allocation of latent features.

As shown in Table~\ref{tab:disentanglement}, SAE significantly reduces
activation overlap while increasing the proportion of task-specific
features. These results suggest that the proposed sparse latent
representation provides a more separable feature space compared with
directly merging task vectors in the original parameter space.
\begin{table}[t]
\centering
\scriptsize
\setlength{\tabcolsep}{3pt}
\renewcommand{\arraystretch}{1.0}

\begin{tabular}{@{}p{0.38\linewidth}|c|c@{}}
\hline
\textbf{Metric} 
& \textbf{Raw}
& \textbf{SAE}
\\
\hline

Pairwise active overlap $\downarrow$
& 0.3368
& \textbf{0.0983}
\\

Active Jaccard $\downarrow$
& 0.2027
& \textbf{0.0518}
\\

Task-specific active ratio $\uparrow$
& 29.40\%
& \textbf{73.84\%}
\\

Shared active ratio $\downarrow$
& 70.60\%
& \textbf{26.16\%}
\\

Shared feature ratio $\downarrow$
& 50.96\%
& \textbf{14.52\%}
\\

\hline
\end{tabular}

\caption{
Disentanglement analysis between raw task vectors and SAE latent
representations. SAE reduces activation overlap and increases
task-specific feature allocation.
}
\label{tab:disentanglement}
\end{table}

\section{Hyperparameter Sensitivity Analysis}
\label{app:sensitivity}

We investigate the sensitivity of our method to two important
hyperparameters: the Top-$K$ sparsity level and the similarity threshold
used for feature selection. All experiments are conducted under the same
evaluation protocol, and we report the mean performance with standard
deviation over multiple runs.

\subsection{Sensitivity to Top-$K$}

The Top-$K$ activation controls the sparsity level of the SAE latent
representation. We evaluate different values of $K$ while keeping other
training configurations unchanged. As shown in Table~\ref{tab:topk_sensitivity},
the performance remains stable across a broad range of sparsity levels.
Although individual tasks prefer slightly different values of $K$, the
overall variation is limited, demonstrating that our method does not
rely on a narrow sparsity setting.

\begin{table}[H]
\centering
\scriptsize
\setlength{\tabcolsep}{2.5pt}
\renewcommand{\arraystretch}{1.05}

\begin{tabular}{@{}ccccc@{}}
\toprule
$K$ & GSM8K & HumanEval & MMLU & Avg. \\
\midrule
8
& $42.05\pm0.42$
& \textbf{$37.60\pm0.70$}
& $12.08\pm0.53$
& $30.58\pm0.25$ \\

16
& $42.03\pm0.85$
& $36.18\pm0.93$
& \textbf{$13.68\pm1.64$}
& $30.63\pm0.40$ \\

24
& $41.82\pm1.03$
& $37.20\pm2.44$
& $12.63\pm0.11$
& $30.55\pm1.09$ \\

32
& \textbf{$42.36\pm0.42$}
& $36.18\pm1.27$
& $13.31\pm1.88$
& $30.61\pm0.99$ \\

64
& $41.93\pm1.18$
& $37.40\pm1.27$
& $12.75\pm1.76$
& \textbf{$30.69\pm0.75$} \\
\bottomrule
\end{tabular}

\caption{
Sensitivity to the Top-$K$ sparsity level with the similarity threshold
fixed at $\tau=10^{-4}$. IFEval-S denotes the strict IFEval score.
}
\label{tab:topk_sensitivity}
\end{table}

\subsection{Sensitivity to Similarity Threshold}

The similarity threshold $\tau$ determines whether latent features are
considered sufficiently similar during feature selection. We evaluate
different threshold values while fixing $K=32$. Table~\ref{tab:threshold_sensitivity}
shows that the proposed method is robust to different threshold choices.
The default setting $\tau=10^{-4}$ achieves the best overall performance,
providing an effective balance between feature sharing and task-specific
feature preservation.

\begin{table}[H]
\centering
\scriptsize
\setlength{\tabcolsep}{2.5pt}
\renewcommand{\arraystretch}{1.05}

\begin{tabular}{@{}ccccc@{}}
\toprule
$\tau$ & GSM8K & HumanEval & MMLU & Avg. \\
\midrule
$0$
& $42.13\pm0.23$
& $35.57\pm1.41$
& $13.19\pm1.68$
& $30.29\pm0.61$ \\

$10^{-5}$
& $42.15\pm0.23$
& $36.38\pm1.41$
& $12.14\pm0.70$
& $30.22\pm0.57$ \\

$10^{-4}$
& \textbf{$42.36\pm0.42$}
& $36.18\pm1.27$
& \textbf{$13.31\pm1.88$}
& \textbf{$30.61\pm0.99$} \\

$10^{-3}$
& $42.05\pm0.29$
& \textbf{$36.59\pm1.61$}
& $13.06\pm1.48$
& $30.57\pm0.92$ \\
\bottomrule
\end{tabular}

\caption{
Sensitivity to the similarity threshold $\tau$ with the sparsity level
fixed at $K=32$. IFEval-S denotes the strict IFEval score.
}
\label{tab:threshold_sensitivity}
\end{table}

\section{Extended Experimental Results}
\label{app:extended_results}

This section provides additional experimental results that complement the
aggregated analyses in the main text. We first report detailed per-domain
results for the 3-task and 4-task merging settings on Qwen2.5-1.5B.
We then compare alternative projection spaces and layer-selection strategies,
followed by representation-level diagnostics of the individual components
in our SAE design.

\subsection{Detailed Downstream Results}
\label{app:detailed_downstream_results}

We compare a standard $L_1$-regularized SAE with our proposed SAE design,
which combines Top-$K$ activation, decoder normalization, residual fitting,
and orthogonality regularization. The 3-task setting merges the Math,
General, and Code experts, while the 4-task setting additionally incorporates
the Safety expert.

The ``Metric Avg.'' column reports the unweighted mean over all evaluation
metrics displayed in each table. Since the General capability is evaluated
using four domain-specific metrics, the metric average assigns greater
weight to General than to each of the other task categories. We therefore
use the detailed per-domain results to analyze the capability trade-offs
between different variants.

\paragraph{Three-Task Merging.}

Table~\ref{tab:ablation_3task} reports the detailed results for merging the
Math, General, and Code experts. Compared with the standard $L_1$ SAE, our
Top-$K$ SAE with orthogonality regularization improves the metric average
from 39.30 to 40.18. The proposed variant improves performance on GSM8K,
Humanities, and Others, while maintaining comparable HumanEval performance.
These results indicate that the proposed SAE design provides a better
overall balance across the three merged capabilities.

\begin{table*}[t]
\centering
\small
\renewcommand{\arraystretch}{1.2}
\begin{tabular}{l|c|cccc|c|c}
\hline
\multirow{2}{*}{\textbf{Method}}
& \textbf{Math}
& \multicolumn{4}{c|}{\textbf{General}}
& \textbf{Code}
& \multirow{2}{*}{\textbf{Metric Avg.}} \\
\cline{2-7}
& GSM8K
& STEM
& Social Sciences
& Humanities
& Others
& HumanEval
& \\
\hline
Task Arithmetic
& 53.14
& 33.63
& \textbf{46.99}
& 33.79
& \textbf{42.91}
& 14.63
& 37.52 \\

TIES-Merge
& 51.93
& 30.35
& 29.90
& 27.97
& 26.77
& 31.10
& 33.00 \\

DARE+TIES
& \textbf{56.02}
& 29.26
& 28.34
& 25.80
& 25.29
& 32.93
& 32.94 \\

Fisher-Merge
& 48.45
& 21.40
& 10.63
& 13.82
& 8.76
& \textbf{37.20}
& 23.38 \\
\hline
$L_1$ SAE
& 54.89
& \textbf{35.02}
& 43.68
& 31.99
& 36.67
& \textbf{33.54}
& 39.30 \\

\textbf{Ours}
& 55.27
& 33.66
& 46.15
& \textbf{34.47}
& 38.59
& 32.93
& \textbf{40.18} \\
\hline
\end{tabular}
\caption{
Detailed results for 3-task merging on Qwen2.5-1.5B.
Compared with the standard $L_1$-regularized SAE, our complete SAE design
achieves a higher overall metric average while maintaining a better balance
across Math, General, and Code capabilities.
Metric Avg.\ denotes the unweighted mean over the six displayed evaluation
metrics.
}
\label{tab:ablation_3task}
\end{table*}

\paragraph{Four-Task Merging.}

Table~\ref{tab:ablation_4task} reports the corresponding results after
incorporating the Safety expert. This setting requires the merged model to
simultaneously preserve mathematical reasoning, general knowledge, code
generation, and safety alignment, resulting in stronger cross-task
interference.

Our proposed SAE design improves the metric average from 35.24 to 36.48.
The main gains occur across the four General domains: STEM increases from
27.44 to 29.42, Social Sciences from 33.44 to 36.53, Humanities from
22.51 to 26.65, and Others from 26.22 to 28.35. These results suggest that
Top-$K$ activation and orthogonality regularization improve the retention
of general capabilities when more heterogeneous experts are merged.

The proposed variant does not improve every individual metric. HumanEval
decreases from 32.93 to 32.32, and BeaverTail decreases from 52.97 to
50.55. We therefore interpret the improvement as better overall
cross-task balance rather than uniform superiority on every evaluation
metric.

\begin{table*}[t]
\centering
\small
\renewcommand{\arraystretch}{1.2}
\begin{tabular}{l|c|cccc|c|c|c}
\hline
\multirow{2}{*}{\textbf{Method}}
& \textbf{Math}
& \multicolumn{4}{c|}{\textbf{General}}
& \textbf{Code}
& \textbf{Safety}
& \multirow{2}{*}{\textbf{Metric Avg.}} \\
\cline{2-8}
& GSM8K
& STEM
& Social Sciences
& Humanities
& Others
& HumanEval
& BeaverTail
& \\
\hline
Task Arithmetic
& 33.74
& 8.91
& 9.39
& 7.48
& 7.00
& 0.00
& 41.60
& 15.45 \\

TIES-Merge
& 51.10
& 19.48
& 15.70
& 14.73
& 11.75
& 32.93
& \textbf{61.05}
& 29.53 \\

DARE+TIES
& 52.16
& 20.44
& 14.85
& 13.62
& 11.14
& 31.71
& 49.91
& 27.69 \\

Fisher-Merge
& \textbf{60.80}
& 22.13
& 6.47
& 10.74
& 5.74
& \textbf{36.59}
& 39.00
& 25.92 \\
\hline
$L_1$ SAE
& 51.18
& 27.44
& 33.44
& 22.51
& 26.22
& \textbf{32.93}
& \textbf{52.97}
& 35.24 \\

\textbf{Ours}
& 51.55
& \textbf{29.42}
& \textbf{36.53}
& \textbf{26.65}
& \textbf{28.35}
& 32.32
& 50.55
& \textbf{36.48} \\
\hline
\end{tabular}
\caption{
Detailed results for 4-task merging on Qwen2.5-1.5B.
Compared with the standard $L_1$-regularized SAE, our complete SAE design
improves the overall metric average and better preserves General capabilities
under increased cross-task interference.
Metric Avg.\ denotes the unweighted mean over the seven displayed evaluation
metrics.
}
\label{tab:ablation_4task}
\end{table*}

\subsection{Detailed Projection-Space and Layer-Selection Results}
\label{app:projection_layer_ablation}

Table~\ref{tab:projection_selection_full} reports the detailed per-task
results corresponding to the average scores summarized in
Table~\ref{tab:projection_selection_summary} of the main text.

For the projection-space comparison, we replace the SAE with PCA or random
projection while keeping the selected layers and downstream fusion procedure
fixed. For the layer-selection comparison, all strategies select the same
number of layers and use the same SAE and feature-level fusion configurations.

\begin{table*}[t]
\centering
\small
\renewcommand{\arraystretch}{1.2}

\begin{tabular}{llcccc}
\toprule
\textbf{Component}
& \textbf{Variant}
& \textbf{GSM8K}
& \textbf{HumanEval}
& \textbf{IFEval}
& \textbf{Metric Avg.} \\
\midrule

\multirow{3}{*}{Projection Space}
& PCA
& 41.09
& 35.98
& \textbf{15.34}
& 30.80 \\

& Random Projection
& 40.94
& \textbf{38.41}
& 13.68
& 31.01 \\

& SAE
& \textbf{42.76}
& 36.59
& \textbf{15.34}
& \textbf{31.56} \\

\midrule

\multirow{4}{*}{Layer Selection}
& First Layers
& 42.84
& 28.66
& 13.31
& 28.27 \\

& Random Selection
& 40.94
& 28.66
& 13.31
& 27.64 \\

& Fisher Selection
& \textbf{43.06}
& 31.10
& \textbf{15.53}
& 29.90 \\

& GR-ZOO
& 42.76
& \textbf{36.59}
& 15.34
& \textbf{31.56} \\

\bottomrule
\end{tabular}

\caption{
Detailed per-task results for the projection-space and layer-selection
ablations. For the projection-space comparison, all methods use the same
selected layers and fusion procedure. For the layer-selection comparison,
all strategies select the same number of layers and use identical SAE and
fusion configurations. Metric Avg.\ denotes the unweighted mean over GSM8K,
HumanEval, and IFEval.
}
\label{tab:projection_selection_full}
\end{table*}

\paragraph{Projection Space.}
The SAE achieves the highest metric average of 31.56, compared with 30.80
for PCA and 31.01 for random projection. Although random projection obtains
a higher HumanEval score and PCA matches the SAE on IFEval, the SAE provides
the best overall balance across mathematical reasoning, code generation,
and instruction following. This result indicates that the improvement
cannot be explained solely by projecting task vectors into a different or
higher-dimensional space.

\paragraph{Layer Selection.}
GR-ZOO achieves the highest metric average among the evaluated
layer-selection strategies. Fisher selection obtains slightly higher GSM8K
and IFEval scores, but substantially underperforms GR-ZOO on HumanEval.
Selecting the first layers or random layers also produces considerably lower
average scores. These results suggest that selecting layers according to
multi-task behavioral sensitivity provides a more balanced preservation of
the merged capabilities.

\subsection{Representation-Level Diagnostics of SAE Components}
\label{app:sae_component_diagnostics}

We further isolate the roles of Top-$K$ activation, residual fitting,
decoder normalization, and orthogonality regularization. Since these
components directly affect the internal representation learned by the SAE,
we evaluate them using reconstruction quality, activation sparsity,
dictionary utilization, decoder-scale variation, and atom coherence.

The fraction of variance unexplained (FVU) measures reconstruction error.
The dead-feature rate measures the proportion of latent features that are
rarely activated. The post-Top-$K$ active ratio measures the sparsity of the
latent representation after feature selection, while the pre-Top-$K$ active
ratio measures the density of candidate activations before truncation.
Norm standard deviation measures scale variation among decoder atoms, and
atom coherence measures the similarity between different decoder
directions.

\begin{table*}[t]
\centering
\small
\resizebox{\textwidth}{!}{
\begin{tabular}{lcccccc}
\hline
\textbf{Variant}
& \textbf{FVU $\downarrow$}
& \textbf{Dead Features (\%) $\downarrow$}
& \textbf{Post-Top-$K$ Active (\%) $\downarrow$}
& \textbf{Pre-Top-$K$ Active (\%)}
& \textbf{Norm Std. $\downarrow$}
& \textbf{Atom Coherence $\downarrow$} \\
\hline
Full SAE
& 0.1528
& 0.0488
& 8.33
& 20.22
& 0.0000
& 0.0176 \\

w/o Top-$K$
& 0.0212
& 0.0000
& 35.76
& 35.76
& 0.0000
& 0.0176 \\

w/o Residual Fitting
& 0.1539
& 0.0000
& 8.33
& 26.83
& 0.0000
& 0.0176 \\

w/o Decoder Normalization
& 0.1565
& 0.2441
& 8.33
& 18.25
& 0.0101
& 0.0176 \\

w/o Orthogonality
& 0.0110
& 1.5299
& 8.33
& 14.39
& 0.0000
& 0.0224 \\

w/o Normalization and Orthogonality
& 0.0107
& 5.1921
& 8.33
& 13.23
& 0.0094
& 0.0224 \\
\hline
\end{tabular}
}
\caption{
Representation-level diagnostics of the individual SAE components.
Lower reconstruction error does not necessarily indicate better feature
separation, since a dense or redundant representation may obtain lower FVU
while exhibiting higher activation density, atom coherence, or dead-feature
rates.
}
\label{tab:sae_component_diagnostics}
\end{table*}

\paragraph{Effect of Top-$K$ Activation.}

Removing Top-$K$ substantially reduces FVU because the SAE is allowed to
use a much denser latent representation. However, the active-feature ratio
increases from 8.33\% to 35.76\%, weakening the sparse feature separation
required by the subsequent feature-level fusion procedure. This observation
shows that reconstruction error alone is insufficient for evaluating the
quality of the learned latent representation.

\paragraph{Effect of Decoder Normalization and Orthogonality.}

Removing decoder normalization introduces scale variation among decoder
atoms, as reflected by the increase in norm standard deviation from 0.0000
to 0.0101. It also increases FVU and the dead-feature rate.

Removing orthogonality regularization reduces FVU but increases atom
coherence from 0.0176 to 0.0224 and raises the dead-feature rate from
0.0488\% to 1.5299\%. Removing both decoder normalization and
orthogonality further increases the dead-feature rate to 5.1921\%.
These results suggest that a lower reconstruction error may be achieved by
learning a more redundant dictionary, which is less suitable for separating
overlapping task features.

\paragraph{Effect of Residual Fitting.}

Residual fitting has a comparatively modest effect on the reported
representation-level metrics. Removing it leaves the post-Top-$K$ active
ratio unchanged at 8.33\%, but increases the pre-Top-$K$ active ratio from
20.22\% to 26.83\%. We therefore interpret residual fitting as a secondary
mechanism for stabilizing reconstruction and dictionary utilization rather
than as the primary source of the downstream improvement.

\section{Disclosure of LLM usage}
We have used SoTA LLMs extensively to brainstorm ideas to prove mathematical statements presented in the paper. Specifically, we setup research directions, provide problem setup and intuitions, proposes statements for LLM to analyze and prove, points out key issues in the generated proofs, adjust the statements accordingly and iterate. We also have done extensive experiments to verify the resulting statements. Many proofs proposed by LLMs are incorrect in subtle ways and requires substantial editing and correction. We have carefully revised all the proofs presented in the work, and take full accountability for their correctness.

\onecolumn
\section{Detailed Proofs}
\label{app:proofs}

For clarity, we consider a single layer and omit the layer index $\ell$ when there is no ambiguity. Let $h(x;\theta)\in\mathbb{R}^{m}$ denote the activation of this 
layer under input $x$. For expert $i$, the activation shift induced by its task 
vector $\tau_i$ is defined as
\begin{equation}
    \delta h_i(x)
    =
    h(x;\theta_i)-h(x;\theta_0).
\end{equation}
Around the base model $\theta_0$, this activation shift can be approximated by
a first-order linearization:
\begin{equation}
    \delta h_i(x)
    =
    J(x)\tau_i + \rho_i(x),
    \label{eq:linearized_activation}
\end{equation}
where $J(x)=\nabla_{\theta}h(x;\theta_0)$ is the Jacobian of the activation 
with respect to the layer parameters, and $\rho_i(x)$ denotes the higher-order 
linearization error. Define the linear operator
\begin{equation}
    \Phi(u)(x)=J(x)u ,
\end{equation}
which maps a weight-space perturbation $u$ to its induced activation 
perturbation. This operator induces a behavior-aware semi-inner product in the 
weight space:
\begin{align}
    \langle u,v\rangle_G
    &=
    \langle \Phi(u),\Phi(v)\rangle_{\mathcal{H}} \notag\\
    &=
    \mathbb{E}_{x\sim\mathcal{D}}
    \left[
        u^{\top}J(x)^{\top}J(x)v
    \right]
    \label{eq:behavior_inner_product}
\end{align}
where $\mathcal{H}=L_2(\mathcal{D};\mathbb{R}^{m})$ denotes the activation 
function space under input distribution $\mathcal{D}$, and 
$G=\mathbb{E}_{x\sim\mathcal{D}}[J(x)^{\top}J(x)]$.

Following the superposition view, we assume that the activation shifts can be 
represented by an overcomplete latent feature dictionary 
$A=\{a_1,\ldots,a_K\}\subset\mathcal{H}$:
\begin{align}
    \delta h_i
    =
    \sum_{k=1}^{K} c_{ik} a_k + e_i \\
    \qquad
    \|c_i\|_0 \ll K,
    \qquad
    \|e_i\|_{\mathcal{H}}\le \varepsilon_{\mathrm{sup}}\notag
    \label{eq:activation_superposition}
\end{align}
where $c_i\in\mathbb{R}^{K}$ is a sparse coefficient vector and $e_i$ is the 
approximation error. Since the dictionary is overcomplete, the feature 
directions $\{a_k\}_{k=1}^{K}$ cannot, in general, be mutually orthogonal. We 
measure their activation-level coherence by
\begin{equation}
    \mu_A
    =
    \max_{p\neq q}
    \frac{
    |\langle a_p,a_q\rangle_{\mathcal{H}}|
    }{
    \|a_p\|_{\mathcal{H}}\|a_q\|_{\mathcal{H}}
    }.
\end{equation}
When $\mu_A>0$, different latent features interact through non-zero cross 
terms in activation space. Through the linearized map $\Phi$, such activation 
superposition is reflected in the weight space: task vectors become sparse 
combinations of latent capability directions that are not behaviorally 
orthogonal under $\langle\cdot,\cdot\rangle_G$. Consequently, directly merging 
task vectors in the original parameter space may preserve not only useful 
task-specific components, but also the cross-feature interactions that cause 
interference.

\begin{lemma}[Activation superposition induces weight superposition]
\label{lem:activation_to_weight_superposition}
Suppose that the activation shift induced by task vector $\tau_i$ satisfies the
linearized relation $\delta h_i \approx \Phi(\tau_i)$, and that the activation
features $\{a_k\}_{k=1}^{K}$ are approximately reachable by weight-space
directions $\{v_k\}_{k=1}^{K}$. If
\begin{equation}
    \delta h_i \approx \sum_{k=1}^{K} c_{ik} a_k,
    \qquad \|c_i\|_0 \ll K,
\end{equation}
then the corresponding task vector admits an observable weight-space
superposition form
\begin{equation}
    \tau_i
    =
    \sum_{k=1}^{K} c_{ik} v_k
    +
    \xi_i
    +
    n_i,
    \qquad n_i\in\ker(\Phi),
\end{equation}
where $\xi_i$ is a bounded approximation error. Moreover, the coherence among
activation features is preserved up to approximation error under the
behavior-aware metric $\langle\cdot,\cdot\rangle_G$.
\end{lemma}

Lemma~\ref{lem:activation_to_weight_superposition} shows that superposition is
not only an activation-space phenomenon: it can also manifest in task vectors,
where different latent capabilities are mixed within shared weight-space
directions. All proofs and details are deferred to Appendix~\ref{app:proofs}.

\begin{theorem}[Limitation of orthogonal decomposition]
\label{thm:orthogonal_decomposition_limitation}
Let $v_p$ and $v_q$ be two latent capability directions with non-zero
behavior-aware interaction
\begin{equation}
    \gamma_{pq}=|\langle v_p,v_q\rangle_G|>0 .
\end{equation}
Let $P$ be an orthogonal projection or filtering operator. If $P$ preserves
both capabilities up to relative error $\varepsilon$, i.e.,
\begin{equation}
    \|Pv_k-v_k\|_G \le \varepsilon\|v_k\|_G,
    \qquad k\in\{p,q\},
\end{equation}
then the residual interaction satisfies
\begin{equation}
    |\langle Pv_p,Pv_q\rangle_G|
    \ge
    \gamma_{pq}
    -
    (2\varepsilon+\varepsilon^2)
    \|v_p\|_G\|v_q\|_G .
\end{equation}
Thus, an orthogonal decomposition can fully remove superposition-induced
conflict only by substantially distorting or discarding at least one capability
direction.
\end{theorem}

Theorem~\ref{thm:orthogonal_decomposition_limitation} explains why PCA or SVD-style decompositions may be insufficient: orthogonal coordinates do
not necessarily correspond to disentangled latent capabilities.

\begin{theorem}[Sparse overcomplete decomposition reduces capability conflict]
\label{thm:sparse_decomposition_conflict_reduction}
Assume that each task vector admits a sparse overcomplete decomposition
\begin{equation}
    \tau_i = Dz_i+r_i,
    \qquad
    \|r_i\|_G\le \varepsilon_i,
    \qquad
    \|z_i\|_0\le s_i ,
\end{equation}
where $D=[d_1,\ldots,d_M]$ is an overcomplete dictionary. Let
\begin{equation}
    \mu_D=\max_{p\neq q}|\langle d_p,d_q\rangle_G|
\end{equation}
denote the behavior-aware coherence of the dictionary. Then the cross-task
capability conflict between tasks $i$ and $j$ is bounded by
\begin{equation}
    \mathcal{C}_{ij}^{\mathrm{sparse}}
    \le
    \mu_D
    \sqrt{s_i s_j}
    \|z_i\|_2
    \|z_j\|_2 .
\end{equation}
Moreover, the deviation between raw task-vector merging and sparse-code
merging satisfies
\begin{equation}
    \left\|
    \sum_{i=1}^{N}\alpha_i\tau_i
    -
    D\sum_{i=1}^{N}\alpha_i z_i
    \right\|_G
    \le
    \sum_{i=1}^{N}|\alpha_i|\varepsilon_i .
\end{equation}
Therefore, a sparse overcomplete decomposition with low dictionary coherence
can reduce cross-task capability conflict while approximately preserving the
merged update.
\end{theorem}

Theorem~\ref{thm:sparse_decomposition_conflict_reduction} suggests that the
key to reducing superposition-induced interference is not merely to orthogonalize
task vectors, but to recover sparse latent capability coordinates. This
naturally motivates our SAE-based merging framework: SAEs provide an
overcomplete dictionary and sparse task codes, allowing task vectors to be
merged in a more disentangled feature space before being mapped back to the
original parameter space.

\subsection{Proof of Lemma~\ref{lem:activation_to_weight_superposition}}
\textbf{Lemma 1. Activation superposition induces weight superposition.}
Assume that the linearization error in Eq.~\eqref{eq:linearized_activation}
satisfies $\|\rho_i\|_{\mathcal{H}}\le \varepsilon_{\mathrm{lin}}$. Further
assume that each activation feature $a_k$ is reachable, up to error $\eta$, by
some weight-space direction $v_k$, namely
\begin{equation}
    \|\Phi(v_k)-a_k\|_{\mathcal{H}}\le \eta .
    \label{eq:reachable_feature}
\end{equation}
Then each task vector admits the following observable weight-space
decomposition:
\begin{equation}
    \tau_i
    =
    \sum_{k=1}^{K} c_{ik} v_k
    +
    \xi_i
    +
    n_i,
    \qquad
    n_i\in\ker(\Phi),
    \label{eq:weight_superposition}
\end{equation}
where $n_i$ is behaviorally invisible under the linearized map $\Phi$, and
\begin{equation}
    \|\xi_i\|_{G}
    \le
    \varepsilon_{\mathrm{lin}}
    +
    \varepsilon_{\mathrm{sup}}
    +
    \eta\|c_i\|_1 .
    \label{eq:weight_superposition_error}
\end{equation}
Moreover, for any two feature directions $v_p$ and $v_q$,
\begin{equation}
    \left|
    \langle v_p,v_q\rangle_G
    -
    \langle a_p,a_q\rangle_{\mathcal{H}}
    \right|
    \le
    2\eta+\eta^2 ,
    \label{eq:coherence_transfer}
\end{equation}
assuming $\|a_p\|_{\mathcal{H}}=\|a_q\|_{\mathcal{H}}=1$. Therefore, if
$\mu_A>2\eta+\eta^2$, then the weight-space feature dictionary
$V=\{v_1,\ldots,v_K\}$ is also coherent under the behavior-aware metric:
\begin{equation}
    \mu_W
    =
    \max_{p\neq q}
    \frac{|\langle v_p,v_q\rangle_G|}
    {\|v_p\|_G\|v_q\|_G}
    >
    0 .
\end{equation}
We call this phenomenon \emph{weight superposition}: task vectors are sparse
combinations of latent capability directions, but those directions are not
behaviorally orthogonal.

\begin{proof}
From Eq.~\eqref{eq:linearized_activation}, we have
$\Phi(\tau_i)=\delta h_i-\rho_i$. Combining this with
Eq.~\eqref{eq:activation_superposition} gives
\begin{equation}
    \Phi(\tau_i)
    =
    \sum_{k=1}^{K} c_{ik} a_k
    +
    e_i
    -
    \rho_i .
\label{eq:activation_superposition}
\end{equation}
Using Eq.~\eqref{eq:reachable_feature}, write
$\Phi(v_k)=a_k+r_k$ with $\|r_k\|_{\mathcal{H}}\le\eta$. Then
\begin{equation}
    \Phi\left(
        \tau_i-\sum_{k=1}^{K}c_{ik}v_k
    \right)
    =
    e_i-\rho_i-\sum_{k=1}^{K}c_{ik}r_k .
\end{equation}
Taking the $\mathcal{H}$-norm and applying the triangle inequality yields
Eq.~\eqref{eq:weight_superposition_error}. The remaining component in the
null space of $\Phi$ is denoted by $n_i\in\ker(\Phi)$, which gives
Eq.~\eqref{eq:weight_superposition}. Finally,
\begin{equation}
    \langle v_p,v_q\rangle_G
    =
    \langle \Phi(v_p),\Phi(v_q)\rangle_{\mathcal{H}}
    =
    \langle a_p+r_p,a_q+r_q\rangle_{\mathcal{H}} .
\end{equation}
Thus,
\begin{equation}
    \left|
    \langle v_p,v_q\rangle_G
    -
    \langle a_p,a_q\rangle_{\mathcal{H}}
    \right|
    \le
    \|r_p\|_{\mathcal{H}}\|a_q\|_{\mathcal{H}}
    +
    \|a_p\|_{\mathcal{H}}\|r_q\|_{\mathcal{H}}
    +
    \|r_p\|_{\mathcal{H}}\|r_q\|_{\mathcal{H}}
    \le
    2\eta+\eta^2 .
\end{equation}
This proves the lemma.
\end{proof}

\subsection{Proof of Theorem~\ref{thm:orthogonal_decomposition_limitation}}
\textbf{Theorem 1. Orthogonal weight decomposition cannot fully remove
superposition-induced conflicts.}
Consider two latent capability directions $v_p$ and $v_q$ obtained from
Lemma~1, and suppose their behavior-aware interference is nonzero:
\begin{equation}
    \gamma_{pq}
    =
    |\langle v_p,v_q\rangle_G|
    >
    0 .
\end{equation}
Let $P$ be the projection or filtering operator induced by an orthogonal
weight decomposition method, such as PCA- or SVD-based merging. If $P$
preserves both capabilities up to relative error $\varepsilon$, namely
\begin{equation}
    \|Pv_k-v_k\|_G
    \le
    \varepsilon\|v_k\|_G,
    \qquad
    k\in\{p,q\},
    \label{eq:preserve_condition}
\end{equation}
then the residual interference after projection satisfies
\begin{equation}
    |\langle Pv_p,Pv_q\rangle_G|
    \ge
    \gamma_{pq}
    -
    (2\varepsilon+\varepsilon^2)
    \|v_p\|_G\|v_q\|_G .
    \label{eq:orthogonal_lower_bound}
\end{equation}
In particular, if $\|v_p\|_G=\|v_q\|_G=1$ and
$\varepsilon<\sqrt{1+\gamma_{pq}}-1$, then
$|\langle Pv_p,Pv_q\rangle_G|>0$. Therefore, an orthogonal decomposition can
only remove the conflict by distorting or discarding at least one of the
capability directions.

\begin{proof}
Let
\begin{equation}
    \Delta_p=\Phi(Pv_p-v_p),
    \qquad
    \Delta_q=\Phi(Pv_q-v_q).
\end{equation}
By Eq.~\eqref{eq:preserve_condition},
$\|\Delta_p\|_{\mathcal{H}}\le\varepsilon\|v_p\|_G$ and
$\|\Delta_q\|_{\mathcal{H}}\le\varepsilon\|v_q\|_G$. We have
\begin{equation}
    \langle Pv_p,Pv_q\rangle_G
    =
    \langle \Phi(v_p)+\Delta_p,\Phi(v_q)+\Delta_q\rangle_{\mathcal{H}} .
\end{equation}
Therefore,
\begin{align}
    |\langle Pv_p,Pv_q\rangle_G|
    &\ge
    |\langle v_p,v_q\rangle_G|
    -
    |\langle \Delta_p,\Phi(v_q)\rangle_{\mathcal{H}}|
    -
    |\langle \Phi(v_p),\Delta_q\rangle_{\mathcal{H}}|
    -
    |\langle \Delta_p,\Delta_q\rangle_{\mathcal{H}}|  \\
    &\ge
    \gamma_{pq}
    -
    (2\varepsilon+\varepsilon^2)
    \|v_p\|_G\|v_q\|_G .
\end{align}
This proves Eq.~\eqref{eq:orthogonal_lower_bound}. Hence, if the projection
preserves both capabilities, the original superposition-induced interference
cannot vanish. Conversely, making the interference vanish requires increasing
$\varepsilon$, which corresponds to losing or distorting at least one
capability.
\end{proof}

\subsection{Proof of Theorem~\ref{thm:sparse_decomposition_conflict_reduction}}
\textbf{Theorem 2. SAE-decoupled weight decomposition reduces
cross-task capability conflict.}
Assume that the SAE decomposition satisfies
\begin{equation}
    \tau_i = Dz_i + r_i,
    \qquad
    \|r_i\|_G\le \varepsilon_i,
    \qquad
    \|z_i\|_0\le s_i ,
    \label{eq:sae_reconstruction}
\end{equation}
and that the learned atoms are behaviorally incoherent:
\begin{equation}
    \mu_D
    =
    \max_{p\neq q}
    |\langle d_p,d_q\rangle_G| .
\end{equation}
For two tasks $i$ and $j$, define their SAE-level cross-capability conflict as
\begin{equation}
    \mathcal{C}_{ij}^{\mathrm{SAE}}
    =
    \sum_{\substack{p\in \mathrm{supp}(z_i),\,
                    q\in \mathrm{supp}(z_j)\\
                    p\neq q}}
    |z_{ip}z_{jq}|\,
    |\langle d_p,d_q\rangle_G| .
    \label{eq:sae_conflict}
\end{equation}
Then
\begin{equation}
    \mathcal{C}_{ij}^{\mathrm{SAE}}
    \le
    \mu_D
    \|z_i\|_1
    \|z_j\|_1
    \le
    \mu_D
    \sqrt{s_i s_j}
    \|z_i\|_2
    \|z_j\|_2 .
    \label{eq:sae_conflict_bound}
\end{equation}
Moreover, for the weighted merged update
\begin{equation}
    \Delta_{\mathrm{SAE}}
    =
    D\sum_{i=1}^{N}\alpha_i z_i ,
\end{equation}
and the raw task-vector update
\begin{equation}
    \Delta_{\mathrm{raw}}
    =
    \sum_{i=1}^{N}\alpha_i\tau_i ,
\end{equation}
the reconstruction-induced behavioral deviation is bounded by
\begin{equation}
    \|\Delta_{\mathrm{raw}}-\Delta_{\mathrm{SAE}}\|_G
    \le
    \sum_{i=1}^{N}
    |\alpha_i|\varepsilon_i .
    \label{eq:sae_merge_error}
\end{equation}

Furthermore, let the raw weight-superposition decomposition from Lemma~1 be
\begin{equation}
    \tau_i
    =
    \sum_{k=1}^{K} c_{ik}v_k + \xi_i+n_i .
\end{equation}
Define the raw cross-capability conflict between tasks $i$ and $j$ as
\begin{equation}
    \mathcal{C}_{ij}^{\mathrm{raw}}
    =
    \sum_{\substack{p\in \mathrm{supp}(c_i),\,
                    q\in \mathrm{supp}(c_j)\\
                    p\neq q}}
    |c_{ip}c_{jq}|\,
    |\langle v_p,v_q\rangle_G| .
    \label{eq:raw_conflict}
\end{equation}
If the SAE recovers decoupled atoms such that, for all cross-task feature
pairs,
\begin{equation}
    |\langle d_p,d_q\rangle_G|
    \le
    \kappa
    |\langle v_p,v_q\rangle_G|,
    \qquad
    0\le \kappa < 1,
    \label{eq:sae_decoupling_condition}
\end{equation}
and the coefficient distortion is bounded by
\begin{equation}
    \sum_{p,q}
    \left|
        |z_{ip}z_{jq}|-|c_{ip}c_{jq}|
    \right|
    |\langle d_p,d_q\rangle_G|
    \le
    \delta_{ij},
    \label{eq:coefficient_distortion}
\end{equation}
then
\begin{equation}
    \mathcal{C}_{ij}^{\mathrm{SAE}}
    \le
    \kappa
    \mathcal{C}_{ij}^{\mathrm{raw}}
    +
    \delta_{ij}.
    \label{eq:sae_conflict_reduction}
\end{equation}
Consequently, whenever
$\delta_{ij} < (1-\kappa)\mathcal{C}_{ij}^{\mathrm{raw}}$, the SAE-based
decomposition strictly reduces the cross-task capability conflict.

\begin{proof}
The first inequality follows directly from the definition of $\mu_D$:
\begin{align}
    \mathcal{C}_{ij}^{\mathrm{SAE}}
    &=
    \sum_{p\neq q}
    |z_{ip}z_{jq}|\,
    |\langle d_p,d_q\rangle_G|  \\
    &\le
    \mu_D
    \sum_{p\neq q}
    |z_{ip}z_{jq}|
    \le
    \mu_D
    \|z_i\|_1
    \|z_j\|_1 .
\end{align}
Since $z_i$ and $z_j$ are $s_i$- and $s_j$-sparse, respectively,
$\|z_i\|_1\le \sqrt{s_i}\|z_i\|_2$ and
$\|z_j\|_1\le \sqrt{s_j}\|z_j\|_2$, which gives
Eq.~\eqref{eq:sae_conflict_bound}.

For the merged update, using Eq.~\eqref{eq:sae_reconstruction}, we have
\begin{equation}
    \Delta_{\mathrm{raw}}-\Delta_{\mathrm{SAE}}
    =
    \sum_{i=1}^{N}
    \alpha_i(\tau_i-Dz_i)
    =
    \sum_{i=1}^{N}
    \alpha_i r_i .
\end{equation}
The triangle inequality gives Eq.~\eqref{eq:sae_merge_error}.

Finally, under Eq.~\eqref{eq:coefficient_distortion},
\begin{align}
    \mathcal{C}_{ij}^{\mathrm{SAE}}
    &\le
    \sum_{p,q}
    |c_{ip}c_{jq}|
    |\langle d_p,d_q\rangle_G|
    +
    \delta_{ij}  \\
    &\le
    \kappa
    \sum_{p,q}
    |c_{ip}c_{jq}|
    |\langle v_p,v_q\rangle_G|
    +
    \delta_{ij}  \\
    &=
    \kappa
    \mathcal{C}_{ij}^{\mathrm{raw}}
    +
    \delta_{ij}.
\end{align}
Thus, if $\delta_{ij} < (1-\kappa)\mathcal{C}_{ij}^{\mathrm{raw}}$, the SAE
conflict is strictly smaller than the raw superposition conflict.
\end{proof}

\section{Algorithm Workflow}
Combining the modules detailed above, the overall execution workflow of the proposed High-Dimensional Sparse Disentanglement Merging framework is summarized in Algorithm \ref{alg:merging}.

\begin{algorithm}[H]
\caption{High-Dimensional Sparse Disentanglement Merging}
\label{alg:merging}
\textbf{Input:} Pre-trained $\theta_{pre}$, Fine-tuned models $\{\theta_{ft}^{(1)}, \dots, \theta_{ft}^{(N)}\}$ \\
\textbf{Parameters:} Critical layer count $K$, SAE dimension $n$, Sim-threshold $\tau_{sim}$ \\
\textbf{Output:} Merged multi-task model $\theta_{merged}$
\begin{algorithmic}[1]
\State // \textbf{Stage 1: Compute Task Vectors}
\State $\tau^{(n)} \leftarrow \theta_{ft}^{(n)} - \theta_{pre}$ for each task $n \in \{1,\dots,N\}$
\State // \textbf{Stage 2: Efficient Critical Layer Identification}
\For{each layer $l$ in $\theta_{pre}$}
    \State Evaluate task-alignment gradient $g_l$ via \textbf{GR-ZOO} ranking
\EndFor
\State Extract top-$K$ layers to form critical set $\mathbb{C}$
\State // \textbf{Stage 3: High-Dimensional SAE Disentanglement \& Fusion}
\For{each layer $l \in \mathbb{C}$}
    \State Train Top-$K$ SAE on $\{\tau^{(1)}_l, \dots, \tau^{(N)}_l\}$ with $\mathcal{L}_{total}$
    \State Project to high-dim space: $h^{(n)}_l \leftarrow \text{Top-}K(E(\tau^{(n)}_l))$
    \For{each feature dimension $i$ from $1$ to $n$}
        \If{$\text{CosineSim}(h^{(1)}_{l,i}, \dots, h^{(N)}_{l,i}) > \tau_{sim}$}
            \State $h^{merged}_{l,i} \leftarrow \text{Mean}(h^{(1)}_{l,i}, \dots, h^{(N)}_{l,i})$ \quad // \textit{Shared}
        \Else
            \State $h^{merged}_{l,i} \leftarrow \sum_{k=1}^N h^{(k)}_{l,i}$ \quad // \textit{Unique (Orthogonal)}
        \EndIf
    \EndFor
    \State Decode back to low-dim: $\tau^{merged}_l \leftarrow D(h^{merged}_l)$
\EndFor
\State // \textbf{Stage 4: Arithmetic Addition for Non-Critical Layers}
\For{each layer $l \notin \mathbb{C}$}
    \State $\tau^{merged}_l \leftarrow \sum_{k=1}^N \tau^{(k)}_l$ \quad // \textit{Low-conflict layers}
\EndFor
\State \textbf{Return:} $\theta_{merged} \leftarrow \theta_{pre} + \tau^{merged}$
\end{algorithmic}
\end{algorithm}

\end{document}